\documentclass[conference]{IEEEtran}
\usepackage[T1]{fontenc}
\usepackage{cite}
\usepackage{graphicx}
\usepackage{amsmath,amssymb,amsthm}
\usepackage{url}
\usepackage{fancyhdr}
\fancypagestyle{arxivnotice}{%
  \fancyhf{}%
  \fancyhead[C]{\small\textit{Accepted at the 2026 IEEE International Conference on Electro/Information Technology (EIT 2026)}}%
}

\title{Optimality of Sequential Filtering Under Independent Cost and Selectivity Models}

\author{
\IEEEauthorblockN{
Hrishikesh Paranjape,
Abhishek Mandal,
Xian Sun
}
\IEEEauthorblockA{
Meta, USA
}
\thanks{Accepted for publication in the 2026 IEEE International Conference on Electro/Information Technology (EIT 2026). \copyright~2026 IEEE. Personal use of this material is permitted. Permission from IEEE must be obtained for all other uses, in any current or future media, including reprinting/republishing this material for advertising or promotional purposes, creating new collective works, for resale or redistribution to servers or lists, or reuse of any copyrighted component of this work in other works.}
}
\begin{document}
\maketitle
\thispagestyle{arxivnotice}

\begin{abstract}
Sequential filtering pipelines are a common design
pattern in large-scale systems, where a large population of items is
progressively reduced by a sequence of stages that each incur cost.
Despite their prevalence in ranking systems, cascaded machine
learning inference, and fraud detection, filter ordering is often
determined by heuristics without formal guarantees. We formalize
sequential filtering under an expected-cost objective and prove
that, under an independence model, ordering filters by increasing
ratio of cost to rejection probability minimizes expected total
cost. Extensive Monte Carlo simulations show that the optimal
ordering strictly dominates common heuristics across all runs,
both in expectation and across the full distribution of outcomes.
\end{abstract}


\section{Introduction}
Large-scale systems frequently process massive candidate sets through a sequence of increasingly expensive tests. Examples include multi-stage ranking pipelines in search and recommendation systems, cascaded inference architectures in machine learning, and detection pipelines in spam and fraud prevention.

A central systems question is how to order these stages to minimize overall cost while preserving desired accuracy or recall. In practice, pipeline stages are often ordered using simple heuristics, such as placing the cheapest filters first or placing the most selective filters first. While intuitive, such heuristics can be systematically suboptimal because they ignore the interaction between cost and selectivity.

Consider a pipeline in which a highly selective but expensive filter is applied early: although it rejects many items, the cost incurred on the full population may dominate total computation. Conversely, placing a cheap but weak filter early may fail to sufficiently reduce the population, amplifying downstream cost. These trade-offs suggest the need for a principled ordering criterion.

In this work, we study the sequential filtering problem under a commonly assumed independence model and derive a simple, globally optimal ordering rule. We further demonstrate empirically that this ordering not only minimizes expected cost but also dominates common heuristics across the full distribution of outcomes.

\subsection{Contributions}
\begin{itemize}
    \item We formalize an expected-cost model for sequential filtering pipelines.
  \item We prove a simple and globally optimal ordering rule using a pairwise exchange argument.
  \item We empirically validate the theory using large-scale Monte Carlo simulations.
  \item We show distributional dominance using empirical CDFs and dominance scatter plots.
\end{itemize}

\section{Problem Definition}
We consider a set of filters $F=\{1,\ldots,n\}$. Each filter $i$ is characterized by a per-item cost $c_i>0$ and a pass probability $p_i\in[0,1]$. We assume filter outcomes are independent, costs are additive, and pass probabilities are stationary.

Given an ordering $\pi=(i_1,\ldots,i_n)$, the expected total cost is
\begin{equation}
\mathbb{E}[C(\pi)] = \sum_{k=1}^{n} c_{i_k}\prod_{j=1}^{k-1} p_{i_j},
\label{eq:expected_cost}
\end{equation}
where the empty product equals $1$. Our objective is to find an ordering $\pi^\star$ that minimizes $\mathbb{E}[C(\pi)]$.

\section{Optimal Ordering}
\newtheorem{theorem}{Theorem}
\begin{theorem}[Optimal Sequential Filtering Order]
Under the assumptions of Section~II, ordering filters by non-decreasing value of
\begin{equation}
\frac{c_i}{1-p_i}
\label{eq:ratio_rule}
\end{equation}
minimizes the expected total cost among all orderings.
\end{theorem}

\begin{proof}
Consider two adjacent filters $A$ and $B$ with parameters $(c_A,p_A)$ and $(c_B,p_B)$. Applying $A$ then $B$ incurs expected cost $c_A + p_A c_B$, while applying $B$ then $A$ incurs $c_B + p_B c_A$. The former is no worse if and only if
\begin{equation}
c_A(1-p_B) \le c_B(1-p_A),
\end{equation}
which is equivalent to $\frac{c_A}{1-p_A}\le \frac{c_B}{1-p_B}$. Any ordering violating this condition can be locally improved by swapping adjacent filters. Repeatedly applying such swaps yields a globally optimal ordering.
\end{proof}

\section{Experimental Evaluation}
\subsection{Setup}
We evaluate three strategies: (i) \emph{Cost-order}, which sorts filters by increasing cost; (ii) \emph{Pass-rate-order}, which sorts by increasing pass probability; and (iii) \emph{Optimal}, which sorts by increasing $c/(1-p)$. Each experiment samples $n=50$ filters with costs uniformly drawn from $\{1,\ldots,100\}$ and pass probabilities uniformly drawn from $[0.4,1.0]$. We simulate a population of $10^6$ items and repeat the experiment over $R=500$ independent runs.

\subsection{Aggregate Results}
Across all runs, the optimal ordering achieves the lowest total cost. Cost-based ordering is consistently suboptimal but relatively stable, while pass-rate-based ordering exhibits high variance and occasionally extreme costs. These results highlight the risk of relying solely on selectivity without accounting for cost.

\subsection{Dominance Scatter Plot}
The dominance scatter plot provides a strong per-instance guarantee: the absence of points below the diagonal demonstrates that the optimal ordering never incurs higher cost than the baselines under the experimental model.

\subsection{Empirical CDF Analysis}
The empirical CDF (ECDF) highlights distributional behavior beyond averages. The optimal strategy achieves lower cost for a larger fraction of runs at every cost threshold, while the pass-rate heuristic exhibits a heavy-tailed distribution.

\begin{figure}[t]
\centering
\includegraphics[width=\linewidth]{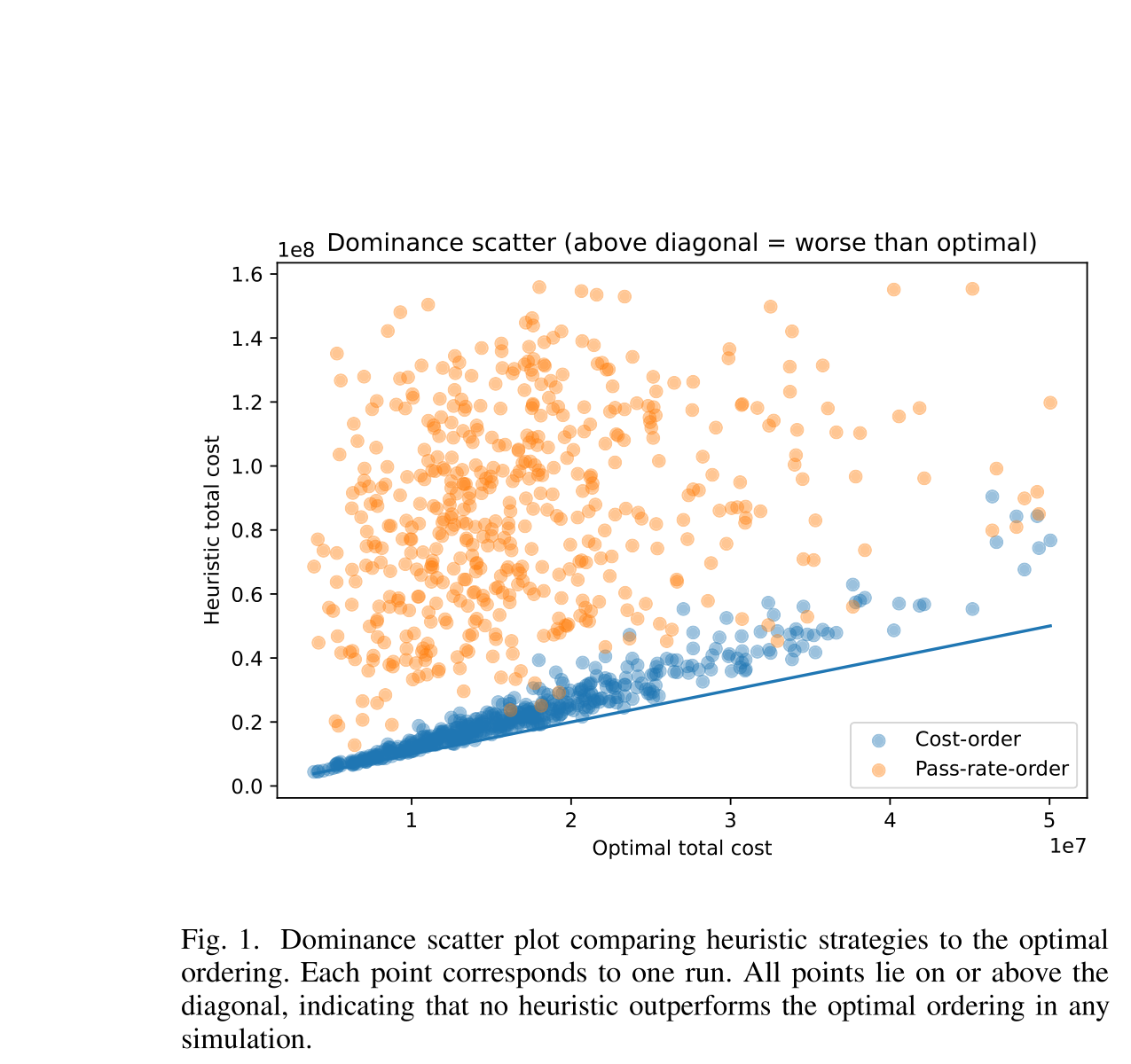}
\caption{Dominance scatter plot comparing heuristic strategies to the optimal ordering. Each point corresponds to one run. All points lie on or above the diagonal, indicating that no heuristic outperforms the optimal ordering in any simulation.}
\label{fig:dominance}
\end{figure}

\begin{figure}[t]
\centering
\includegraphics[width=\linewidth]{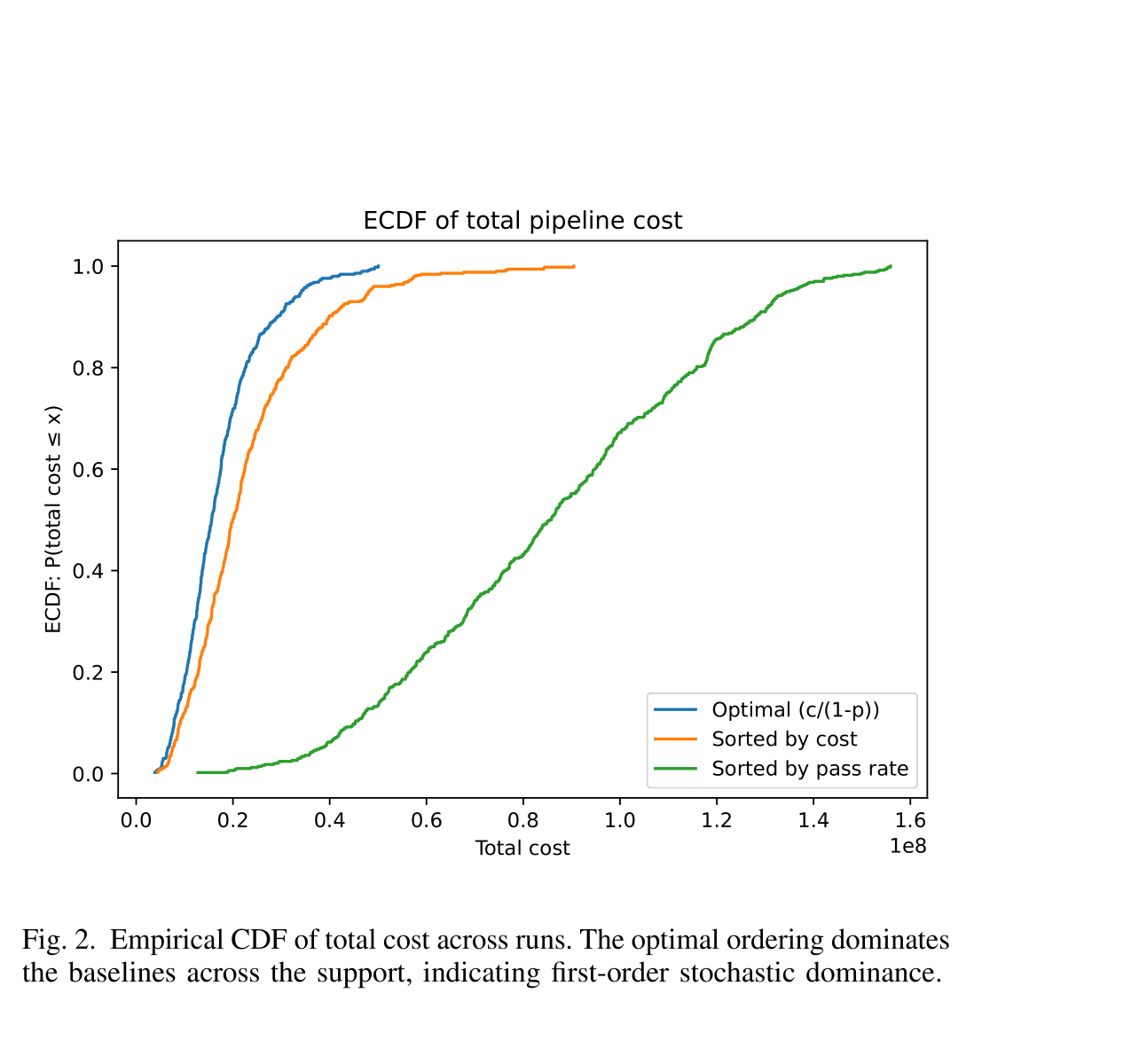}
\caption{Empirical CDF of total cost across runs. The optimal ordering dominates the baselines across the support, indicating first-order stochastic dominance.}
\label{fig:ecdf}
\end{figure}

\section{Related Work}
Our result is closely related to ratio-based scheduling rules, most notably Smith's rule for minimizing weighted completion time on a single machine \cite{smith1956various}. While classical scheduling considers deterministic job completion, our setting models stochastic elimination with additive cost.

Sequential rejection pipelines also appear in cascaded classifiers, such as boosted cascades for object detection \cite{viola2001rapid}, and in modern cascaded inference systems for deep learning with early-exit or cascaded inference architectures \cite{teerapittayanon2016branchynet}. Sequential decision-theoretic foundations for staged testing can be traced back to classical work on sequential analysis \cite{wald1947sequential}.

\section{Discussion and Limitations}
The optimality guarantee relies on independence and stationarity assumptions. In real systems, filters may be correlated, adaptive, or state-dependent. While our empirical results are consistent with the theory under the assumed model, extending the analysis to correlated or adaptive settings is an important direction for future work.

\section{Conclusion}
We presented a formal analysis of sequential filtering pipelines and proved a simple optimal ordering rule under an independence model. Extensive simulations demonstrate that the optimal ordering strictly dominates common heuristics across all runs and across the full distribution of costs. These results provide both theoretical guidance and practical insight for designing efficient large-scale filtering systems.

\bibliographystyle{IEEEtran}
\bibliography{references}

\end{document}